\documentclass[11pt]{article}

\usepackage[margin=1in]{geometry}

\usepackage{amsmath,amssymb,amsthm}
\usepackage{natbib}
\usepackage{graphicx}
\usepackage{url}
\usepackage[colorlinks=true,citecolor=blue,linkcolor=blue,urlcolor=blue]{hyperref}

\usepackage{booktabs}
\usepackage{multirow}
\usepackage{tikz}
\usepackage{thmtools}
\usepackage{thm-restate}
\usepackage{xcolor}
\usepackage{microtype}

\usetikzlibrary{arrows.meta, positioning, calc}

\theoremstyle{plain}
\newtheorem{theorem}{Theorem}[section]
\newtheorem{lemma}[theorem]{Lemma}
\newtheorem{corollary}[theorem]{Corollary}
\newtheorem{proposition}[theorem]{Proposition}

\theoremstyle{definition}
\newtheorem{definition}[theorem]{Definition}

\title{Compact Geometric Representations of Hierarchies}

\author{
Prashant Gokhale \\
UW-Madison 
\and
Piotr Indyk \\
MIT 
\and
Yuhao Liu \\
UW-Madison 
\and
Sandeep Silwal \\
UW-Madison 
\and
Tony Chang Wang \\
UW-Madison  
\and
Haike Xu \\
MIT 
}

\date{}

\begin{document}

\maketitle

\begin{abstract}%
Computing geometric representations of data is a cornerstone of modern machine learning, typically achieved by training dual encoders which map queries and documents into a shared embedding space. Recent work of You et al. [NeurIPS '25] has extended this approach to hierarchical retrieval, where relevance is determined by the ancestor-descendant relationships in a Directed Acyclic Graph (DAG). While previous work has shown that valid embeddings exist when the number of descendants is small, these bounds degrade significantly for deep hierarchies, requiring dimensions as large as the total number of nodes.

In this paper, we investigate compact reachability embeddings for more general graph classes and provide theoretical guarantees for representing hierarchies using embeddings whose dimension depends on structural graph parameters. We prove that for any directed tree, there exists a reachability embedding in constant dimension 3, independent of the tree's size or depth. We generalize this result to graphs characterized by treewidth $t$, constructing embeddings of dimension $O(t \log n)$, where $n$ is the number of nodes. Complementing these upper bounds, we provide matching or near-matching lower bounds, showing that dimension $\Omega(n)$ is necessary for general DAGs and $\Omega(t/\log(n/t))$ is required for graphs of treewidth $t$. We also obtain upper and lower bounds parameterized by the number of cross-edges in the DAG. We additionally show that our embeddings can be constructed on real world datasets, and that they give much smaller dimensions in high recall regimes compared to prior embeddings with theoretical guarantees.

\end{abstract}


\section{Introduction}

Computing geometric representations of data, often called embeddings, is one of the key tools in modern machine learning and information retrieval. Such representations are typically computed by training a {\em bi-encoder}, i.e., two mappings $a$ and $b$, such that $a$ maps queries and $b$ maps  data points  into some $d$-dimensional feature space. The mappings are designed such that the data points $p$ that are relevant to a query $q$ are mapped to similar representations $a(q)$ and $b(p)$, and to dissimilar representations otherwise. Formally, $a(q)$ and $b(p)$ are defined as {\em similar} if and only if the inner product $a(q) \cdot b(p)>r$ for some threshold $r$. Thanks to such embeddings,  retrieval is reduced to (approximate) nearest neighbor search in a high-dimensional space, a task for which many efficient algorithms have been developed~\citep{andoni2008near, malkov2018efficient,jayaram2019diskann}.  This approach, often referred to as {\em dense retrieval}, has grown in popularity over the last two decades, exemplified by~\citet{zhan2021optimizing,zhao2024dense,fang2024scaling,chen2024dense, theoreticaltopk, denseretrieval1, denseretrieval2}. 

The accuracy of dense retrieval is mostly determined by the quality of the embeddings. Although such quality embeddings exist for many retrieval tasks, recent work on more complex benchmarks such as BRIGHT~\citep{su2024bright} demonstrates the challenges of this approach when the definition of relevance involves complex reasoning. This motivated several recent works~\citep{guo2019breaking,weller2025theoretical, you2025hierarchical} investigating when, and under what conditions, good bi-encoders exist.

In this paper we study embeddings for {\em hierarchical retrieval}, an important class of retrieval tasks where the documents are organized into an (often implicit) hierarchy that determines the relevance. For example, a document on “Poodles” is relevant to a query about a more general category of “Dogs”, while the reverse is not true in general. Such hierarchies can be modeled as Directed Acyclic Graphs (DAGs) defined over the data: a document is relevant to a query if the query corresponds to an {\em ancestor} of the document in the DAG. Formally, $a$ and $b$ are reachability embeddings if for any pair of nodes $u$ and $v$ in a DAG, $v$ is reachable from $u$ if and only if $a(u) \cdot b(v)>0$.

A recent work~\citep{you2025hierarchical} pioneered the study of embeddings that represent such notion of relevance. In a surprising finding, it shows that if the number of descendants $S(v)$ of any node $v$ in the DAG is bounded by $s$, then reachability embeddings exist even in dimension of roughly $d=O(s \log n)$, where $n$ is the number of documents.  

However, their work leaves open the setting when some nodes $v$ have many descendants. For example, if a hierarchy is a tree, then the bound $s$ corresponds to the height of the tree, which could be as large as $n$. 
Furthermore, the aforementioned theorem in fact holds in {\em any} setting where the number of nodes relevant to any query $q$ is bounded by $s$ - the relevance sets do not need to involve reachability in a  DAG~\citep{guo2019breaking}. This raises the question of whether more fine-grained embeddings exist, especially ones that are parametrized by well-studied graph parameters.

\begin{table}[!h]
\centering
{\renewcommand{\arraystretch}{1.75}
\begin{tabular}{c|c|c}

\textbf{Graph Class} & \textbf{Upper Bound} & \textbf{Lower Bound} \\ \hline
General DAG & \begin{tabular}[c]{@{}c@{}}$O(n)$\\ ~\cite{you2025hierarchical} \end{tabular} & \begin{tabular}[c]{@{}c@{}}$\Omega\left( n \right)$\\ Theorem \ref{lb:dag}\end{tabular} \\ \hline
Treewidth $t$ & \begin{tabular}[c]{@{}c@{}}$O(t \log n)$\\ Theorem \ref{dig}\end{tabular} & \begin{tabular}[c]{@{}c@{}}$\Omega\left( \frac{t}{\log{\frac{n}{t}}} \right)$\\ Theorem \ref{lb:t/logn}\end{tabular} \\ \hline
Cross-Edges $k$ & \begin{tabular}[c]{@{}c@{}}$3 + k$\\ Theorem \ref{dag}\end{tabular} & \begin{tabular}[c]{@{}c@{}}$\Omega\left(\max\left(k^{1/4}, \frac{k \log(n^2/k)}{n \log n} \right)\right)$\\ Theorems \ref{thm:crossedge-lb} and \ref{lb:cross}\end{tabular} \\ \hline
Directed Trees & \begin{tabular}[c]{@{}c@{}}$3$\\ Theorem \ref{tree}\end{tabular} & \begin{tabular}[c]{@{}c@{}}$ 3$\\ Lemma \ref{lem:tree_lb} \end{tabular} \\ 
\end{tabular}
}
\caption{A summary of our upper and lower bounds for the embedding dimension needed to represent reachability in various graph classes.\label{table:results}}
\end{table}

\paragraph{Our results} In this paper we answer the above question in the affirmative. Specifically, we present compact reachability embeddings whose dimension depends on the following graph parameters:
\begin{itemize}
\item {\em Treewidth:} We use the treewidth of the underlying undirected graph. Informally, a graph has treewidth $t$ if it can be decomposed into  balanced components by removing  $t$ nodes (see Section \ref{sec:prelims} for a formal definition). Treewidth is a well-studied graph parameter characterizing how ``tree like'' the graph is; in particular, the treewidth of a tree is $1$. It is also a popular parameter characterizing the complexity of graph algorithms \citep{cygan2015parameterized}, as well as algorithms in the machine learning literature~\citep{elidan2008learning,scanagatta2016learning}. In particular, many real-world graphs representing hierarchical relationships admit small treewidth; for example, for the WordNet dataset~\citep{wordnet} used by \cite{you2025hierarchical}, a simple check on the underlying undirected graph shows that its treewidth is at most $37$.

\item  {\em Cross-Edges}: 
A distinct measure that we also use to measure the complexity of a hierarchy is the number of cross-edges. Given any directed tree $T$ with edges from the DAG $G$, an edge $(u,v)$ is called a {\em cross-edge} w.r.t. $T$ if $(u,v)$ is an edge in $G$ but $v$ is not reachable from $u$ in $T$. This measure is popular in the database community for parameterizing the memory of data structures answering DAG reachability queries (see \cite{wang2006dual,yildirim2010grail,chen2005stack} and the survey \cite{zhang2025indexing} for details). Our work instead uses this natural parameter for characterizing the embedding dimension.

\end{itemize}

Our results are summarized in Table \ref{table:results}. In short, we present reachability embeddings with dimension depending linearly on parameters of interest (treewidth or the number of cross-edges). In the specific case where the graph is a directed tree, we construct embeddings into space of {\em constant} dimension $3$. Note that in this case, prior embeddings given in~\citet{you2025hierarchical} required dimension as high as $\Omega(n)$, where $n$ is the total number of nodes in the graph. We note that dimension $3$ is necessary, even for directed trees of constant size (Lemma \ref{lem:tree_lb}).

Lastly, in Appendix \ref{sec:independence}, we describe two graph families, one having constant tree-width but $\Omega(n)$ number of cross-edges, and the other having $\Omega(n)$ tree-width but a constant number of cross edges, showing that our two upper bounds are incomparable in general.

\section{Preliminaries}\label{sec:prelims} 
We first define reachability embeddings. 

\begin{definition}[Reachability]
    Let $G=(V, E)$ be a directed graph. We say a vertex $v$ is reachable from another vertex $u$ in $G$, denoted by $u \rightsquigarrow_G v$, if and only if there is a directed path from $u$ to $v$ in $G$. We denote by $u \not\rightsquigarrow_G v$ if there is no path from $u$ to $v$ in $G$. We may drop the subscript $G$ when it is clear from context.
\end{definition}

\begin{definition}[Reachability Embedding]
Let $G=(V, E)$ be a directed graph. A reachability embedding of dimension $d$ in $G$ is defined as a pair of two embeddings $a, b : V \mapsto \mathbb{R}^d$ such that for any two vertices $u, v \in V$:
$$
    \langle a_u, b_v \rangle > 0 \iff  u \rightsquigarrow_G v.
$$

\end{definition}

Next we formally introduce treewidth and review relevant properties useful for Theorem \ref{dig}. First we state the definition of a tree cover.

\begin{definition}\label{def:tree_cover}
    Let $G = (V, E)$ be an undirected graph, $T$ be a tree, and $\mathcal{V}=\left(V_t\right)_{t \in T}$ be a family of vertex sets $V_t \subseteq V$ indexed by all the nodes $t$ of $T$ (we also call $\mathcal{V}$ the bags of $V$). Then $(T, \mathcal{V})$ is a tree-decomposition of $G$ if the following hold:
    \begin{itemize}
        \item $V = \cup_{t \in T} V_t$;
        \item For every edge $(u,v) \in E$, there exists a bag containing both $u,v$: that is $u,v \in V_t$ for some $t \in T$;
        \item For any vertex $u \in V$, all the bags containing $u$ (i.e., $\left\{t \in T: u \in V_t\right\}$) induce a connected subtree of $T$.
    \end{itemize}
    When $T$ is a path, the tree-decomposition is called a path-decomposition.
\end{definition}

Intuitively, the treewidth of $G$ is the minimum integer $k$ such that there exists a tree decomposition of $G$ with bags of size at most $k+1$.

\begin{definition}[Treewidth]\label{def:width}
    The width of a tree decomposition of an undirected graph $G$ is one less than the maximum bag size of that tree decomposition of $G$. The treewidth $\operatorname{tw} \left(G\right)$ of $G$ is the minimum width of all tree decompositions of $G$.

\end{definition}

Treewidth has the following properties which we use in the construction of Theorem \ref{dig}.


\begin{theorem} [\citet{robertson1986graph}]\label{sep}
    Let $G=(V,E)$ be an undirected graph with treewidth $t$. Then there exists a vertex set $S\subseteq V$ with $|S|\leq t+1$ such that every connected component of $G[V\setminus S]$ has at most $\frac{1}{2}|V\setminus S|$ vertices. The vertex set $S$ is also referred to as vertex separators.
\end{theorem}

\begin{lemma}[\citet{robertson1986graph}]\label{minor}
    Let $G=(V,E)$ be an undirected graph with treewidth $t$. Any minor $H$ of $G$ has treewidth at most $t$. Furthermore, let $H$ be any subgraph of $G$. Then
    $\operatorname{tw} \left(H\right) \leq  \operatorname{tw} \left(G\right) = t.$
\end{lemma}

%

Next we define the notion of cross-edges. This is a standard definition for the edges whose reachability relationships are not captured by a depth-first search tree of a DAG.

\begin{definition}\label{def:cross}
    Let $G=(V, E)$ be a directed acyclic graph (DAG) and $F = (V, E_F)$ be a directed spanning-forest subgraph of $G$. We say $(u,v)\in E$ is a cross-edge w.r.t. $F$ in $G$ if $u\not\rightsquigarrow_F v$. 
\end{definition}

For simplicity, we abbreviate a cross-edge w.r.t. $F$ in $G$ simply as a cross-edge, where no ambiguity arises. Note that, by Definition \ref{def:cross}, a cross-edge must be in $E-E_F$ (a non-forest edge). Additional preliminaries about standard algorithmic results relating to depth-first search and cross-edges are given in Appendix \ref{sec:dfs}.

The construction of reachability embeddings in some of our theorems is built on the reachability information of certain `special' vertices (e.g. the separator vertices of Theorem \ref{sep}), necessitating the following definition.

\begin{definition}
    Let $G=(V, E)$ be a directed graph. For any vertex $v\in V$, we define the reachability sets $In(v)$ and $Out(v)$ w.r.t. $G$ as:
    \begin{equation*}
        In(v)=\{w\in V|w\rightsquigarrow_G v\} \cup \{v\} \quad \text{and} \quad Out(v)=\{w\in V|v\rightsquigarrow_G w\}\cup \{v\}.
    \end{equation*}
\end{definition}

\section{Reachability Embeddings for Directed Trees} \label{sec:2}

In this section, we study embeddings for the simplest non-trivial setting: rooted directed trees. Note a directed tree is a DAG whose underlying graph is a tree and rooted implies there is a vertex which can reach all other vertices.

\begin{theorem}\label{tree}
Let $T=(V, E)$ be a rooted directed tree. There exists a reachability embedding of dimension $d=3$ for $T$.
\end{theorem}

The key observation is that in a rooted directed tree, reachability information is concisely captured by the discovery/finishing times of the standard depth-first search (DFS) algorithm (see Appendix \ref{sec:additional_prelims} for an overview of DFS). Letting $d_v$ denote the \emph{discovery time} of a vertex $v$ (the first time DFS visits $v$), and $f_v$ denote the \emph{finish time} (when DFS exits $v$), we have the following interval containment result, stated formally in Theorem \ref{Pare}: $v$ is a descendant of $u$ if and only if $[d_v,f_v]\subseteq[d_u,f_u]$. 

Thus, we design vectors $a_u,b_v\in\mathbb{R}^3$ whose inner product equals
$(f_u-d_v)(d_v-d_u)$, which is positive exactly when $d_v$ lies strictly between $d_u$ and $f_u$,
i.e., exactly when $v$ lies in the DFS subtree of $u$.
\begin{proof}[Proof of Theorem \ref{tree}]
Running depth-first search (DFS) on $T$ yields discovery time $d_u$ and the finishing time $f_u$ for any $u\in V$, whose properties are given in Definition \ref{def:dfs}. By Theorem \ref{Pare}, for any $u,v\in V$, $v$ is a descendant of $u$ if and only if $\left[d_v, f_v\right] \subseteq\left[d_u, f_u\right]$.
Thus, we can define the embeddings $a,b:V\to \mathbb{R}^3$ as follows:
\begin{equation*}
a_u = \begin{pmatrix} d_u + f_u, -1, -d_uf_u \end{pmatrix} \quad \text{and} \quad b_u = \begin{pmatrix} d_u, d^2_u, 1 \end{pmatrix},\quad \forall u\in V.
\end{equation*}
Since $\langle a_u,b_v\rangle=(d_u+f_u)d_v-d_v^2-d_uf_u=(f_u-d_v)(d_v-d_u)$, then we have 
\begin{equation*}
    \langle a_u,b_v\rangle>0 \iff d_u<d_v<f_u \text{ or }f_u<d_v<d_u.
\end{equation*}
Furthermore, by Theorem \ref{Pare}, if $d_u<d_v<f_u$, then $\left[d_v, f_v\right] \subseteq\left[d_u, f_u\right]$. Therefore,
\begin{equation*}
    \langle a_u,b_v\rangle>0 \iff \left[d_v, f_v\right] \subseteq\left[d_u, f_u\right]\iff v ~\text{is a descendant of}~ u \iff u\rightsquigarrow_T v.
\end{equation*}

\end{proof}

The construction above is component-wise: it only relies on DFS interval structure within a tree.
As a result, it extends immediately to directed forests by running DFS on one tree component after another and applying the same embedding construction on each tree component since the reachability never holds across components.

\begin{corollary}\label{for}
    Let $F=(V, E)$ be a directed forest. There exists a reachability embedding of dimension $d=3$ in $F$.
\end{corollary}

\begin{proof}
    Again running DFS on $F$ yields discovery time $d_u$ and the finishing time $f_u$ for any $u\in V$. We still use the embeddings $a,b:V\to \mathbb{R}^3$ as described in Theorem~\ref{tree}. For the same reasons as in Theorem~\ref{tree}, we know that for any two vertices $u,v$ in the same tree component of $F$, we can guarantee that:
    \begin{equation*}
        \langle a_u,b_v\rangle>0 \iff u\rightsquigarrow_{F} v.
    \end{equation*}
    On the other hand, for any two vertices $u,v$ in the different tree components of $F$, it is clear that $u\not\rightsquigarrow_{F} v$. By Theorem \ref{Pare}, the intervals $[d_u,f_u]$ and $[d_v,f_v]$ are entirely disjoint. Thus,
    \begin{equation*}
        \langle a_u,b_v\rangle=(f_u-d_v)(d_v-d_u)<0.
    \end{equation*}
    Therefore, this pair of two embeddings is a valid reachability embedding of dimension $d=3$ in $F$.
\end{proof}
We note that results of \cite{delsarte1989low} and \cite{pmlr-v49-alon16} already give a construction of a directed tree which requires a reachability embedding of $\ge 3$; see Lemma \ref{lem:tree_lb}.

\section{Reachability Embeddings for DAGs Parameterized by Cross-Edges}\label{sec:cross}

In Section~\ref{sec:2} we constructed a $3$-dimensional embedding that exactly captures reachability in a directed forest.
For a general DAG, the obstruction is that a spanning forest $F\subseteq G$ may miss directed paths that rely on
non-forest edges (i.e., cross-edges). In this section we show a simple augmentation principle:
each cross-edge that creates a mismatch between reachability in $F$ and in $G$ can be ``repaired'' by introducing one additional coordinate.
As a consequence, we obtain an embedding whose dimension grows linearly with the number of cross-edges.

Informally, the forest embedding certifies reachability whenever a path exists entirely inside $F$.
For any remaining reachable pair $(u,v)$, every directed path from $u$ to $v$ in $G$ must contain at least one cross-edge.
We use the coordinate corresponding to that cross-edge to contribute a large positive term $M^2$,
dominating any negative contribution from the forest part.

\begin{theorem} \label{dag}
Let $G=(V, E)$ be a DAG with a spanning forest $F$. Denote by $E_{\operatorname{cross}}$ the set of all cross-edges, and let $k=|E_{\operatorname{cross}}|$. Then there exists a reachability embedding of $G$ with dimension $d=3+k$.
\end{theorem}


We start from the $3$-dimensional forest embedding and add one coordinate per cross-edge.
The new coordinate for each cross-edge $(x_i,y_i)$ activates exactly for vertices that can reach into $x_i$ and vertices reachable from $y_i$.
Choosing $M$ sufficiently large ensures that any reachable pair not already captured by the forest gains a positive $M^2$ contribution,
while unreachable pairs cannot simultaneously activate any coordinate.

\begin{proof}[Proof of Theorem \ref{dag}]
By Corollary \ref{for}, there exists a reachability embedding of dimension $d'=3$ in $F$, denoted by $a^{F}$ and $b^{F}$. By the definition of reachability embedding, we have that for any two vertices $u, v \in V$:
\begin{equation*}
    \langle a^F_u, b^F_v \rangle > 0 \iff  u \rightsquigarrow_F v.
\end{equation*}
Suppose that $E_{\operatorname{cross}}=\{e_1, \dots, e_k\}$, where $e_i = (x_i,y_i)$. We consider the reachability sets $In(x_i)$ and $Out(y_i)$ w.r.t. $G$ for any $i\in [k]$. 
Let $M$ be a sufficiently large positive constant such that 
\begin{equation*}
    M^2+\min_{u,v\in V} \langle a_u^F,b_v^F\rangle>0.
\end{equation*}
Then we define the embeddings $a^{cross},b^{cross}:V\mapsto \mathbb{R}^k$ as follows:
\begin{equation*}
    a^{cross}_u[i]=\begin{cases} M & \text{if } u \in In(x_i) \\ 0 & \text{otherwise} \end{cases}\quad \text{and} \quad b^{cross}_u[i] = \begin{cases} M & \text{if } u \in Out(y_i) \\ 0 & \text{otherwise} \end{cases},\quad \forall u\in V,i\in [k],
\end{equation*}
where $a^{cross}_u[i]$ and $b^{cross}_u[i]$ denote the $i$-th coordinates of $a^{cross}_u$ and $b^{cross}_u$ respectively.

Finally, we define the embedding $a:V\mapsto \mathbb{R}^{3+k}$ as a concatenation of $a^F$ and $a^{cross}$:
$$
    a_u=(a_u^F;a_u^{cross}), \quad\text{for all }u\in V.
$$
Similarly, we define the embedding $b:V\mapsto \mathbb{R}^{3+k}$ with
$$
    b_u=(b_u^F;b_u^{cross}), \quad\text{for all }u\in V.
$$

The embeddings $a,b: V\mapsto \mathbb{R}^{3+k}$ defined above are a reachability embedding in $G$. Notice that
$$
    \langle a_u, b_v \rangle= \langle a^F_u, b^F_v \rangle+\langle a^{cross}_u, b^{cross}_v \rangle.
$$

    There are three cases to be considered:

\begin{itemize}
    \item \textbf{ Case 1: $u\rightsquigarrow_F v$.} The base $3$-dimensional embedding on the forest certifies reachability:
    $\langle a_u^F,b_v^F\rangle>0$.
    The additional cross-edge coordinates can only increase the inner product since
    $\langle a_u^{\text{cross}}, b_v^{\text{cross}}\rangle\ge 0$.
    Hence $\langle a_u,b_v\rangle>0$.
    Finally, any path in $F$ is also a path in $G$, so $u\rightsquigarrow_F v$ implies $u\rightsquigarrow_G v$.

    \item \textbf{ Case 2: $u\not\rightsquigarrow_F v$ but $u\rightsquigarrow_G v$.} Fix any directed path $P$ from $u$ to $v$ in $G$.
    Then the path $P$ contains a cross-edge $e_i=(x_i,y_i)$ such that
    $u\rightsquigarrow_G x_i$ and $y_i\rightsquigarrow_G v$; equivalently, $u\in \mathrm{In}(x_i)$ and $v\in \mathrm{Out}(y_i)$ \footnote{Let $P=(u=v_0\to v_1\to\cdots\to v_\ell=v)$ be a directed path in $G$.
    If $u\not\rightsquigarrow_F v$, let $j$ be the smallest index such that $u\not\rightsquigarrow_F v_j$.
    Then $u\rightsquigarrow_F v_{j-1}$ but $u\not\rightsquigarrow_F v_j$. So the edge $(v_{j-1},v_j)$ is a cross-edge by definition,
    and its head is $v_j$.}.
    By the definition of the cross-edge coordinates, this implies $a_u^{\text{cross}}[i]=b_v^{\text{cross}}[i]=M$, and hence
    $$
        \langle a_u^{\text{cross}}, b_v^{\text{cross}}\rangle \;\ge\; a_u^{\text{cross}}[i]\cdot b_v^{\text{cross}}[i] \;=\; M^2.
    $$
    Since $\langle a_u^F,b_v^F\rangle \ge \min_{p,q\in V}\langle a_p^F,b_q^F\rangle$, our choice of $M$ ensures
    $$
        \langle a_u,b_v\rangle \;=\; \langle a_u^F,b_v^F\rangle + \langle a_u^{\text{cross}}, b_v^{\text{cross}}\rangle \;>\; 0.
    $$
    
    \item \textbf{ Case 3: $u\not\rightsquigarrow_G v$.} Then in particular $u\not\rightsquigarrow_F v$, and the forest embedding gives
    $\langle a_u^F,b_v^F\rangle\le 0$ by Corollary \ref{for}.
    It remains to argue that the cross-edge coordinates do not create a false positive.
    We claim $\langle a_u^{\text{cross}}, b_v^{\text{cross}}\rangle=0$:
    indeed, if there is an index $i$ with $a_u^{\text{cross}}[i]=b_v^{\text{cross}}[i]=M$, then
    $u\in \mathrm{In}(x_i)$ and $v\in \mathrm{Out}(y_i)$, so $u\rightsquigarrow_G x_i$, $(x_i,y_i)\in E$, and $y_i\rightsquigarrow_G v$,
    which would imply $u\rightsquigarrow_G v$, a contradiction.
    Therefore $\langle a_u^{\text{cross}}, b_v^{\text{cross}}\rangle=0$, and hence
    $$
    \langle a_u,b_v\rangle=\langle a_u^F,b_v^F\rangle+\langle a_u^{\text{cross}}, b_v^{\text{cross}}\rangle \le 0.
    $$

\end{itemize}

\end{proof}

The proof above is the only place where we use the definition of cross-edges: any path not realized inside the forest must contain an edge that leaves the forest reachability relation.


\section{Reachability Embeddings Parameterized by Treewidth} \label{sec:tw}
This section proves our main upper bound. We note that it holds for any directed graph, not necessarily DAGs.

\begin{theorem}\label{dig}
Let $G$ be a directed graph with $n$ vertices whose underlying undirected graph has treewidth $t$.
Then $G$ admits a reachability embedding of dimension $d(n)=O(t\log n)$.
\end{theorem}

The construction for Theorem \ref{dig} is recursive and relies on balanced vertex separators implied by bounded treewidth (Theorem~\ref{sep}). At each step, we remove a separator $S$ of size at most $t+1$ so that the remaining graph breaks into components of size at most half.
We then group these components into two sides $A$ and $B$ of size at most $2n/3$, embed $G[A]$ and $G[B]$ recursively. Naively, one can hope to simply combine the embeddings of the two components via concatenation, but this may 
create \emph{false positives}: pairs of vertices on two different components with no reachability relationships but positive inner product. Thus, we carefully add a small number of additional coordinates to  (i) eliminate false positives across the split and
(ii) certify reachability that is realized via vertices in the separator.

We begin with a simple combinatorial partitioning lemma. After removing a small separator $S$, the graph
$G[V\setminus S]$ decomposes into connected components. Then Lemma~\ref{balapart} will allow us to group the component
sizes into two balanced sides while keeping entire components intact so that no edges cross between the two sides.

\begin{restatable}{lemma}{balaseparestate}\label{balapart}
Let $x_1, \dots, x_k$ be positive integers such that $\sum_{i=1}^k x_i = n$ and $\max_{1\leq i\leq k}x_i \le n/2$. Then there exists a $2$-partition $[I_A,I_B]$ of $[k]$ (i.e., $I_A\cup I_B=[k]$ and $I_A\cap I_B=\emptyset$) such that:
\[
S_A=\sum_{i \in I_A} x_i \le \frac{2n}{3} \quad \text{and} \quad S_B=\sum_{i \in I_B} x_i \le \frac{2n}{3}.
\]
\end{restatable}

\begin{proof}
See Appendix~\ref{App:combpart}.
\end{proof}

Next we combine Lemma~\ref{balapart} with the standard balanced-separator guarantee for bounded treewidth.
Theorem~\ref{sep} ensures that removing at most $t+1$ vertices breaks the graph into components of size at most half.
We then apply Lemma~\ref{balapart} to merge these components into two groups of size at most $2n/3$,
obtaining a separation $(A,B,S)$ suitable for recursion.

\begin{lemma}\label{balasepa}
Let $G=(V,E)$ be an undirected graph on $n$ vertices with treewidth $t$.
Then there exists a separator $S\subseteq V$ with $|S|\leq t+1$ and a partition $V\setminus S = A \sqcup B$ such that:
\begin{enumerate}
    \item $|A|\le \frac{2n}{3}$ and $|B|\le \frac{2n}{3}$,
    \item there are no edges between $A$ and $B$ in $G[V\setminus S]$.
\end{enumerate}
\end{lemma}

\begin{proof}
By Theorem~\ref{sep}, there exists a separator $S\subseteq V$ with $|S|\le t+1$ such that every connected component of
$G[V\setminus S]$ has size at most $|V\setminus S|/2$.
Let these component sizes be $x_1,\dots,x_k$, so that $\sum_{i=1}^k x_i = m$ where $m:=|V\setminus S|$, and
$\max_i x_i \leq m/2$.
Applying Lemma~\ref{balapart}, we can partition the components into two groups whose total sizes are each at most $2m/3$.
Let $A$ (resp.\ $B$) be the union of components in the first (resp.\ second) group.
Then $|A|\leq 2m/3$ and $|B|\leq 2m/3$, and there are no edges between $A$ and $B$ in $G[V\setminus S]$ because they are unions
of distinct connected components. Finally, since $m\leq n$, we also have $ |A|,|B|\leq 2n/3.$
\end{proof}

To prove Theorem \ref{dig}, given a balanced separation $V\setminus S = A\sqcup B$, we embed $G[A]$ and $G[B]$ recursively. We then add two separation coordinates that impose a large negative contribution on cross-pairs $(A,B)$,
preventing false positives across the split. Finally, we add $|S|$ connection coordinates to contribute a large
positive term exactly when a directed path from $u$ to $v$ can be realized via some separator vertex in $S$.

\begin{proof}[Proof of Theorem \ref{dig}]
We prove the claim by induction on $n=|V|$. The base case $n=1$ is immediate. Let $f(n)$ be the number of dimensions required for $n$. Assume the theorem holds for all directed graphs on fewer than $n$ vertices whose underlying undirected graph has treewidth at most~$t$.
Let $G=(V,E)$ be a directed graph on $n$ vertices, and suppose the underlying undirected graph has treewidth~$t$.
By Lemma~\ref{balasepa}, there exists a separator $S=\{s_1,\dots,s_k\}$ with $k\le t+1$ and a partition
$V\setminus S = A \sqcup B$ such that $|A|,|B|\le 2n/3$ and there are no edges between $A$ and $B$ in $G[V\setminus S]$.
By Lemma \ref{minor}, the induced subgraphs $G[A]$ and $G[B]$ have treewidth at most~$t$.
By the inductive hypothesis, $G[A]$ and $G[B]$ admit reachability embeddings $(a^A,b^A)$ and $(a^B,b^B)$ of dimension $f(2n/3)$.
We construct embeddings $a,b:V\mapsto \mathbb{R}^{f(2n/3)+2+k}$ by concatenating three blocks of coordinates.

\paragraph{Coordinate construction.}
\begin{enumerate}
    \item 
        \textbf{ Recursion Coordinates.}
        Define $a^{\mathrm{rec}},b^{\mathrm{rec}}:V\mapsto\mathbb{R}^{f(2n/3)}$ by
        $$
            a_u^{\mathrm{rec}}=
            \begin{cases}
            a_u^A & u\in A,\\
            a_u^B & u\in B,\\
            \mathbf{0} & u\in S,
            \end{cases}
            \qquad
            b_u^{\mathrm{rec}}=
            \begin{cases}
            b_u^A & u\in A,\\
            b_u^B & u\in B,\\
            \mathbf{0} & u\in S.
            \end{cases}
        $$

        In particular, for $u,v\in A$ we have $\langle a_u^{\mathrm{rec}},b_v^{\mathrm{rec}}\rangle>0 \iff u\rightsquigarrow_{G[A]} v$,
        and similarly for $u,v\in B$. If $u\in S$ or $v\in S$ then $\langle a_u^{\mathrm{rec}},b_v^{\mathrm{rec}}\rangle=0$.
    
    \item \textbf{ Separation Coordinates.} 
        Let $M_1>0$ satisfy
        $$
                M_1 > \max_{p,q\in V}\ \langle a_p^{\mathrm{rec}}, b_q^{\mathrm{rec}}\rangle.
        $$
        Define $a^{\mathrm{sep}},b^{\mathrm{sep}}:V\mapsto\mathbb{R}^{2}$ by
        $$
                a_u^{\mathrm{sep}}=
                \begin{cases}
                (1,0) & u\in A,\\
                (0,1) & u\in B,\\
                (0,0) & u\in S,
                \end{cases}
                \qquad
                b_u^{\mathrm{sep}}=
                \begin{cases}
                (0,-M_1) & u\in A,\\
                (-M_1,0) & u\in B,\\
                (0,0) & u\in S.
                \end{cases}
        $$
        Then for any $u,v\in V$,
        $$
            \langle a_u^{\mathrm{sep}},b_v^{\mathrm{sep}}\rangle =
            \begin{cases}
            -M_1 & \text{if } (u\in A \text{ and } v\in B)\ \text{or}\ (u\in B \text{ and } v\in A),\\
            0 & \text{otherwise}.
            \end{cases}
        $$

    \item  \textbf{ Connection Coordinates.} 
        Let $M_2>0$ satisfy
        $$
            M_2^2 > M_1 - \min_{p,q\in V}\ \langle a_p^{\mathrm{rec}}, b_q^{\mathrm{rec}}\rangle.
        $$
        Define $a^{\mathrm{conn}},b^{\mathrm{conn}}:V\mapsto \mathbb{R}^{k}$ by, for each $i\in[k]$,
        $$
        a_u^{\mathrm{conn}}[i]=
        \begin{cases}
        M_2 & \text{if } u\rightsquigarrow_G s_i,\\
        0 & \text{otherwise},
        \end{cases}
        \qquad
        b_u^{\mathrm{conn}}[i]=
        \begin{cases}
        M_2 & \text{if } s_i\rightsquigarrow_G u,\\
        0 & \text{otherwise}.
        \end{cases}
        $$
        Thus $\langle a_u^{\mathrm{conn}}, b_v^{\mathrm{conn}}\rangle>0$ iff there exists $s_i\in S$ with
        $u\rightsquigarrow_G s_i$ and $s_i\rightsquigarrow_G v$.
\end{enumerate}

Finally, we concatenate the three coordinates and define the reachability embeddings: for all $u \in V$,
$$
    a_u := (a_u^{\mathrm{rec}}; a_u^{\mathrm{sep}}; a_u^{\mathrm{conn}}),
    \qquad
    b_u := (b_u^{\mathrm{rec}}; b_u^{\mathrm{sep}}; b_u^{\mathrm{conn}}),
$$
so that
$$
    \langle a_u,b_v\rangle
    =
    \langle a_u^{\mathrm{rec}}, b_v^{\mathrm{rec}}\rangle
    +
    \langle a_u^{\mathrm{sep}}, b_v^{\mathrm{sep}}\rangle
    +
    \langle a_u^{\mathrm{conn}}, b_v^{\mathrm{conn}}\rangle.
$$

\paragraph{Correctness.}

We prove that $\langle a_u,b_v\rangle>0$ if and only if $u\rightsquigarrow_G v$.
Call a pair $(u,v)$ \emph{reachable $u \rightsquigarrow_G v$ via $S$} if there exists a separator vertex $s_i\in S$ with
$u\rightsquigarrow_G s_i$ and $s_i \rightsquigarrow_G v$.
Every directed $u\to v$ path either passes through $S$ (yielding reachability via $S$) or avoids $S$ altogether.
In the latter case the path is contained in $G[V\setminus S]$; because there are no edges between $A$ and $B$ in this subgraph,
we must have $u$ and $v$ on the same side, i.e., $u\rightsquigarrow_{G[A]} v$ when $u,v\in A$ or $u\rightsquigarrow_{G[B]} v$ when $u,v\in B$.

\begin{itemize}
    \item \textbf{ Case 1:  $u \rightsquigarrow_G v$ via $S$.} 
        Then $\langle a_u^{\mathrm{conn}}, b_v^{\mathrm{conn}}\rangle \geq M_2^2$.
        Moreover, $\langle a_u^{\mathrm{sep}}, b_v^{\mathrm{sep}}\rangle \geq -M_1$ always, even if $u,v \in A$ (resp. \  $u,v \in B$). Hence, the inner product of their reachability embeddings:
        \begin{equation*}
            \begin{split}
                \left\langle a_u, b_v\right\rangle &=\left\langle a_u^{\text {rec }}, b_v^{\text {rec }}\right\rangle+\left\langle a_u^{\text {sep }}, b_v^{\text {sep }}\right\rangle+\left\langle a_u^{\text {conn }}, b_v^{\text {conn }}\right\rangle\\
                &\geq \min_{p, q \in V}\left\langle a_p^{r e c}, b_q^{r e c}\right\rangle - M_1 +M_2^2  > 0
            \end{split}
        \end{equation*}
        by the choice of $M_2$.

    \item \textbf{ Case 2: $u \rightsquigarrow_G v$ not via $S$.} 
        Then any $u\to v$ path lies entirely in $V\setminus S$. Since there are no edges between $A$ and $B$ in $G[V\setminus S]$,
        we must have either $u,v\in A$ or $u,v\in B$. Assume $u,v\in A$ (the other case is identical). Then
        $u\rightsquigarrow_{G[A]} v$, so $\langle a_u^{\mathrm{rec}}, b_v^{\mathrm{rec}}\rangle>0$ by the inductive embedding on $G[A]$,
        and $\langle a_u^{\mathrm{sep}}, b_v^{\mathrm{sep}}\rangle=0$. Also $\langle a_u^{\mathrm{conn}}, b_v^{\mathrm{conn}}\rangle=0$ by assumption.
        Therefore $\langle a_u,b_v\rangle>0$.

    \item \textbf{ Case 3: $u \not\rightsquigarrow_G v$.}
        Then in particular $u$ cannot reach $v$ via any separator vertex, so $\langle a_u^{\mathrm{conn}}, b_v^{\mathrm{conn}}\rangle=0$.
        We show $\langle a_u^{\mathrm{rec}}, b_v^{\mathrm{rec}}\rangle + \langle a_u^{\mathrm{sep}}, b_v^{\mathrm{sep}}\rangle \le 0$.
        
        \smallskip
            \emph{(i) $u$ and $v$ lie on different sides.}
            If $u\in A, v\in B$ or $u\in B, v\in A$, then $\langle a_u^{\mathrm{sep}}, b_v^{\mathrm{sep}}\rangle=-M_1$, and hence
            \[
            \langle a_u,b_v\rangle
            \le
            \max_{p,q\in V}\langle a_p^{\mathrm{rec}}, b_q^{\mathrm{rec}}\rangle - M_1
            <0
            \]
            by the choice of $M_1$.

        \smallskip
            \emph{(ii) $u$ and $v$ lie on the same side.}
            Assume $u,v\in A$ (the case $u,v\in B$ is identical). Then $\langle a_u^{\mathrm{sep}}, b_v^{\mathrm{sep}}\rangle=0$.
            Since $u\not\rightsquigarrow_G v$, we also have $u\not\rightsquigarrow_{G[A]} v$, and therefore
            $\langle a_u^{\mathrm{rec}}, b_v^{\mathrm{rec}}\rangle\le 0$ by correctness of the recursive embedding on $G[A]$.
            Thus $\langle a_u,b_v\rangle\le 0$.

        \smallskip
            \emph{(iii) at least one vertex lies in $S$.}
            If $u\in S$ then $a_u^{\mathrm{rec}}=a_u^{\mathrm{sep}}=\mathbf{0}$, and if $v\in S$ then $b_v^{\mathrm{rec}}=b_v^{\mathrm{sep}}=\mathbf{0}$.
            Together with $\langle a_u^{\mathrm{conn}}, b_v^{\mathrm{conn}}\rangle=0$, this implies $\langle a_u,b_v\rangle=0$.
\end{itemize}

\paragraph{Dimension bound.}
The construction uses $f(2n/3)$ recursive coordinates, $2$ separation coordinates, and $k\leq t+1$ connection coordinates, hence
$$
    f(n) \le f(2n/3) + 2 + (t+1) = f(2n/3) + O(t).
$$
Unrolling the recurrence yields $f(n)=O(t\log n)$, completing the proof.
\end{proof}

\section{Lower Bound on Embedding Dimension}\label{sec:lb}

This section establishes lower bounds on the dimension required by reachability embeddings. We present lower bounds for general DAGs (Theorem \ref{lb:dag}), for graphs with treewidth $t$ (Theorem~\ref{lb:t/logn}), and for DAGs parameterized by the number of cross-edges (Theorems~\ref{thm:crossedge-lb} and~\ref{lb:cross}).

We start with a general lower bound for arbitrary DAGs, then refine it for graph families with bounded treewidth and for graph families with a prescribed number of cross-edges.
Following~\cite{weller2025theoretical}, our main tool is a reduction from embedding dimension to the \emph{sign rank} of the reachability matrix,
which allows us to invoke known counting bounds for low sign-rank sign matrices.

\begin{definition}[Sign rank]
For $x\in\mathbb{R}$, define $\operatorname{sign}(x)=+1$ if $x>0$ and $\operatorname{sign}(x)=-1$ otherwise.
For a matrix $M\in\{\pm 1\}^{n\times n}$, its sign rank, denoted by $\operatorname{rank}_{\pm}(M)$, is the minimum rank of a real matrix $R\in\mathbb{R}^{n\times n}$
such that $M_{ij}=\operatorname{sign}(R_{ij})$ for all $i,j$.
\end{definition}

Suppose $G=(V,E)$ with $|V|=n$ admits a reachability embedding in dimension $d$,
given by maps $a,b:V\mapsto\mathbb{R}^d$ satisfying the convention
\begin{equation}\label{eq:embed-sign-conv}
    \langle a_u,b_v\rangle>0 \ \Longleftrightarrow\ u\rightsquigarrow_G v,
\end{equation}
and thus $\langle a_u,b_v\rangle\leq 0$ whenever $u\not\rightsquigarrow_G v$.
Stacking the vectors into matrices $A,B\in\mathbb{R}^{n\times d}$ and defining $R:=AB^\top$,
we obtain an $n\times n$ real matrix whose entry $R_{uv}$ equals $\langle a_u,b_v\rangle$.
Let $M\in\{\pm1\}^{n\times n}$ be the induced reachability sign matrix $M_{uv}:=\operatorname{sign}(R_{uv})$.
By~\eqref{eq:embed-sign-conv}, $M_{uv}=+1$ if and only if $u\rightsquigarrow_G v$.
Crucially, the factorization $R=AB^\top$ implies $\operatorname{rank}(R)\leq d$ because both $A$ and $B$ have at most $d$ columns.
Therefore the reachability sign matrix $M$ is realized by a rank-$\leq d$ real matrix, and hence
$$
    \operatorname{rank}_{\pm}(M)\ \leq\ \operatorname{rank}(R)\ \leq\ d.
$$
Any lower bound on $\operatorname{rank}_{\pm}(M)$ immediately yields a lower bound on the embedding dimension $d$. We also rely on the following counting bound on the number of matrices with low sign rank.

\begin{lemma}[Lemma 24 of \citet{pmlr-v49-alon16}, Lemma 1.1 of \cite{signranklowerbound}]\label{lem:signrank}
    Let $d\leq n/2$. The number of sign matrices $M\in\{\pm1\}^{n\times n}$ with $\operatorname{rank}_{\pm}(M)\leq d$ is at most $\left(O\left ( n/d\right) \right)^{2dn}  \leq 2^{O\left(dn\log{\frac{n}{d}}\right)} \leq 2^{O\left(dn\log{n}\right)}$.
\end{lemma}

\subsection{General Lower Bound}

The next theorem follows from a counting argument:
(1) Lemma~\ref{lem:signrank} upper bounds how many reachability sign matrices can be represented in dimension $d$,
while (2) an explicit family of DAGs yields exponentially many distinct reachability sign matrices.

\begin{theorem}[Combinatorial Lower Bound]\label{lb:dag}
    Any reachability embedding scheme for DAGs on $n$ vertices requires dimension $d=\Omega(n)$ in the worst case.
\end{theorem}

\begin{proof}
   Let $d$ be the embedding dimension. Fix $d\leq n/2$ and let $\mathcal{M}_d$ be the set of $n\times n$ sign matrices that arise as the sign pattern
    $\operatorname{sign}(AB^\top)$ of some embedding $a,b:V\mapsto \mathbb{R}^d$ (equivalently, matrices of sign rank at most $d$).
    By Lemma~\ref{lem:signrank}, we have
    \begin{equation}
        \label{eq:upper-count}
        |\mathcal{M}_d|\ \leq \ 2^{O\left(dn\log{\frac{n}{d}}\right)} .
    \end{equation}

    We now lower bound the number of distinct reachability matrices that can occur for DAGs on $n$ vertices by an explicit family.
    Assume for simplicity that $n$ is even and partition the vertex set into two parts
    $L$ and $R$ with $|L|=|R|=n/2$.
    Consider directed bipartite graphs $G$ that only have edges from $L$ to $R$ (and no edges inside $L$ or inside $R$).
    Such a graph is a DAG, and moreover every directed path has length at most $1$. Hence for $u\in L$ and $v\in R$ we have
    $$
        u\rightsquigarrow_G v \quad\Longleftrightarrow\quad (u,v)\in E.
    $$
    Therefore, distinct edge sets $E\subseteq L\times R$ yield distinct reachability matrices. There are $|L\times R|=(n/2)^2=n^2/4$ potential edges from $L$ to $R$, each chosen independently, so this family contains $2^{n^2/4}$ distinct reachability matrices.
    In particular, any class $\mathcal{M}_d$ that represents all reachability matrices of DAGs on $n$ vertices must satisfy
    \begin{equation}\label{eq:lower-count}
        |\mathcal{M}_d|\ \geq\ 2^{n^2/4}.
    \end{equation}

    Combining \eqref{eq:upper-count} and \eqref{eq:lower-count} gives
    $2^{O\left(dn\log{\frac{n}{d}}\right)} \geq 2^{n^2/4}$, and taking logarithms yields
    $O\left(dn\log{\frac{n}{d}}\right)\ge \Omega(n^2)$. Rearranging, and solving for $d$ proves $d=\Omega(n)$.
\end{proof}

The same asymptotic lower bound also follows from classical counting results for partial orders.
In particular, \citet[Theorem~1]{kleitman1970number} show that the number of labeled partial orders on $n$ elements is $2^{n^2/4+o(n^2)}$.
Since reachability in a DAG is a partial order, this yields a corresponding lower bound on the number of reachability sign matrices.
Plugging this stronger counting bound into the sign-rank argument above again gives $d=\Omega(n)$.

\subsection{Treewidth-Parameterized Lower Bound}

We now consider the case where the input graphs have bounded treewidth. We construct a specific family of graphs with treewidth $t$ that generates a large number of distinct reachability matrices. 

\begin{definition}
    Let $n$ and $t$ be integers with $1 \le t < n$. We define $\mathcal{G}_{n,t}$ to be the family of directed graphs on vertex set $V = \{1, \dots, n\}$ such that a graph $G=(V, E)$ is in $\mathcal{G}_{n,t}$ if and only if every edge $(u, v) \in E$ satisfies: (1) $u$ is odd and $v$ is even, (2)  $u < v$, and (3) $v - u \le t$.
\end{definition}

We also state the definition of pathwidth which will be useful in lower bounding the treewidth of graphs in $\mathcal{G}_{n,t}$.

\begin{definition}[Pathwidth]
    The width of a path decomposition of an undirected graph $G$ is one less than the maximum bag size of that path decomposition of $G$. The pathwidth $\operatorname{pw} \left(G\right)$ of $G$ is the minimum width of all path decompositions of $G$.
\end{definition}

\begin{lemma}
    Any graph $G \in \mathcal{G}_{n,t}$ has treewidth at most $t$.
\end{lemma}

\begin{proof}
    We show that $G$ has pathwidth at most $t$. Consider the sequence of subsets $X_1, \dots, X_{n-t}$ where $X_i = \{i, i+1, \dots, i+t\}$. For any edge $(u, v) \in E$, we have $u < v \le u+t$, so both endpoints are contained in the set $X_u$. Furthermore, for any vertex $v$, the indices $i$ such that $v \in X_i$ form the interval $[\max(1, v-t), v]$, which is contiguous. Since each set has size $t+1$, the pathwidth (and thus treewidth) is at most $(t+1)-1 = t$.
\end{proof}

\begin{theorem}\label{lb:t/logn}
    For the embedding problem on DAGs with treewidth $t$, the required dimension is $d = \Omega\left(t / \log{\frac{n}{t}} \right)$.
\end{theorem}

\begin{proof}
    Let $d$ be the embedding dimension. Consider the graphs in $\mathcal{G}_{n,t}$. Since edges only go from odd vertices to even vertices, the graph is bipartite and maximum path length is 1. Therefore, for all distinct edge sets, the corresponding reachability relations are distinct.
    
    We now count the number of possible graphs in $\mathcal{G}_{n,t}$. For each odd vertex $u$, there are approximately $t/2$ valid even vertices $v$ such that $u < v \le u+t$. Since there are $n/2$ odd vertices, the total number of possible edges is roughly $(n/2) \cdot (t/2) = nt/4$.  Each potential edge can be chosen independently, so the number of distinct graphs in $\mathcal{G}_{n,t}$ is at least $2^{nt/4}$. Using Lemma \ref{lem:signrank}, the number of sign matrices of rank $d$ is at most $2^{O\left(dn\log{\frac{n}{d}}\right)}$. Rearranging, and solving for $d$ we get,
    \begin{equation*}
        2^{O\left(dn\log{\frac{n}{d}}\right)} \ge 2^{nt/4} \implies O\left(dn \log{\frac{n}{d}}\right) \ge \Omega(nt) \implies d \ge \Omega\left(\frac{t}{\log{\frac{n}{t}}}\right). 
    \end{equation*} \end{proof}

\subsection{Cross-Edge--Parameterized Lower Bounds}\label{sec:additional_lb}
In this subsection we give two distinct lower bounds for how the required embedding dimension must scale with the number of cross-edges. Recall the definition of cross-edges.
Let $F = (V, E_F)$ be a directed spanning-forest subgraph of $G$, then the cross-edges w.r.t. $F$ are all the edges $(u,v)$ where $u$ cannot reach $v$ in $F$.

The first lower bound, based on counting the number of matrices of a particular sign rank via Lemma \ref{lem:signrank}, is given below. The first lower bound is useful in the regime when number of cross-edges is large $(\gg n)$, and the second one is better suited for small values of the number of cross-edges.
\begin{theorem}
\label{thm:crossedge-lb}
    Let $n$ be an even number.
    For every $k\in\{0,1,\dots, \frac{n^2 -2n}{4}\}$, there exists a family $\mathcal{F}_{n,k}$ of DAGs on $n$ vertices such that:
    (i) each $G\in\mathcal{F}_{n,k}$ satisfies $k(G)=k$, and
    (ii) $\mathcal{F}_{n,k}$ contains at least $\binom{\,\Theta(n^2)\,}{k}$
 distinct reachability sign matrices.
    Consequently, any reachability embedding scheme that succeeds on all graphs in $\mathcal{F}_{n,k}$ must use dimension
    $$
    d \;=\; \Omega\left(\frac{k\log(n^2/k)}{n\log n}\right).
    $$
\end{theorem}
\begin{proof}
    Let $m:=n/2$ and partition the vertex set into
    $L=\{\ell_1,\dots,\ell_m\}$ and $R=\{r_1,\dots,r_m\}$.
    Define the directed edge set
    $$
        E_F := \{(\ell_i,r_i): i\in[m]\},
    $$
    and let $U := (L\times R)\setminus E_F$, so $|U|= \frac{n^2 -2n}{4}$.
    For each subset $S\subseteq U$ with $|S|=k$, define the bipartite DAG
    $$
        G_S := (L\cup R,\ E_S),
        \qquad
        E_S := E_F\cup S.
    $$
    Since all edges go from $L$ to $R$, every vertex in $R$ has out-degree $0$.
    Hence every directed path has length at most $1$, so reachability coincides with adjacency:
    for $u\in L$ and $v\in R$,
    $$
        u\rightsquigarrow_{G_S} v \iff (u,v)\in E_S.
    $$ 
    Therefore, distinct choices of $S$ yield distinct reachability relations, and hence distinct reachability sign matrices.

    Since $F = (V,E_F)$ is a directed spanning-forest subgraph of $G_S$, 
    every edge in $E_S\setminus E_F$ is a cross-edge:
    $$
        E_{\mathrm{cross}}(G_S)=E_S\setminus E_F = S,
        \qquad\text{and thus}\qquad
        k(G_S)=|S|=k.
    $$
    It follows that $\mathcal{F}_{n,k}:=\{G_S: S\subseteq U,\ |S|=k\}$ has size $\binom{\Theta \left(n^2\right)}{k}$, every graph in the family has $k$ cross-edges and induces distinct reachability sign matrix.
    Now suppose there is a reachability embedding of dimension $d$ for every graph in $\mathcal{F}_{n,k}$.
    Then all $\binom{\Theta \left(n^2\right)}{k}$ corresponding reachability sign matrices have sign rank at most $d$.
    By Lemma~\ref{lem:signrank}, the number of $n\times n$ sign matrices of sign rank at most $d$ is at most $2^{O(dn\log \frac{n}d)}$.
    Therefore,
    $$
    \binom{\Theta \left(n^2\right)}{k} \;\leq\; 2^{O\left(dn\log{\frac{n}{d}}\right)},
    $$
    and rearranging and using the fact that $ \binom{\Theta \left(n^2\right)}{k} \ge \left(\frac{\Theta(n^2)}{k}\right)^k$ gives $$  d=\Omega\!\left(\frac{k\log(n^2/k)}{n\log n }\right). $$
\end{proof}



For our second lower bound, we deploy a different strategy. Instead of employing a counting argument to count the number of matrices of a particular sign rank (via Lemma \ref{lem:signrank}), we directly construct a family of graphs with large sign rank. In particular, we construct a DAG whose reachability sign matrix $M$ includes the Sylvester–Hadamard matrix $H_{t}$ as its submatrix. The Sylvester–Hadamard matrix $H_{t}$ has spectral norm exactly $\sqrt{t}$, which allows us to lower-bound the sign-matrix of $H_{t}$ by utilizing Forster's Theorem (Theorem \ref{l:forster}). Note that the sign-rank of a matrix is at least the sign-rank of any of its submatrices. Thus, we show a lower bound of the sign-rank of $M$, which immediately yields a lower bound on the embedding dimension $d$ (see the discussion in Section \ref{sec:lb}). Our second lower bound result of the section is the following.
\begin{theorem}\label{lb:cross}
For every sufficiently large $k$, there exists a DAG with $k$ cross-edges such that any reachability embedding for this graph requires dimension $d  = \Omega(k^{1/4})$.
\end{theorem}

For the sake of completeness, we state Forster's Theorem (shown as Theorem~\ref{l:forster}) and the properties of Sylvester–Hadamard matrix.
\begin{theorem}[\cite{forster2002linear}]\label{l:forster}
    For a matrix $M\in \{-1,1\}^{n\times n}$, $rank_{\pm}M\geq n/\|M\|$, where $\|M\|$ represents the spectral norm of $M$.
\end{theorem}

\begin{lemma}[\cite{horadam2012hadamard}]\label{l:hadamard}
    For the Sylvester–Hadamard matrix $H_{t}$ of order $t$, where $t=2^k$ for some positive integer $k$, we have the following properties:
    \begin{enumerate}
        \item The first row and first column of $H_{t}$ consist entirely of $+1$ entries.
        \item Each row (except the first) and each column (except the first) has exactly $t/2$ entries equal to $+1$ and $t/2$ entries equal to $-1$.
        \item The spectral norm of $H_{t}$ is $\|H_{t}\|=\sqrt{t}$.
    \end{enumerate}
\end{lemma}

Furthermore, we need the property that the sign-rank of a matrix is at least the sign-rank of any of its submatrices, the correctness of which can be derived from the definition of sign-rank.
\begin{lemma}\label{l:sign_sub}
    Given a matrix $M\in \{-1,1\}^{n\times n}$. Let $M'$ be a submatrix of $M$. Then $rank_{\pm}M\geq rank_{\pm}M'$.
\end{lemma}
\begin{proof}
    Suppose that $rank_{\pm}M=r$. From the definition of sign-rank, we know that there exists a real matrix $R$ such that $M_{ij}=sign(R_{ij})$ for all $i,j$ and $rank(R)=r$. Let $R'$ be the submatrix of $R$ such that $M'_{i'j'}=R'_{i'j'}$ for all $i',j'$. Since $R'$ is the submatrix of $R$, then we have $rank(R')\leq rank(R)\leq r$. By the definition of sign-rank, we have $rank_{\pm}M'\leq rank(R')\leq r=rank_{\pm}M$.
\end{proof}

The correctness of Theorem \ref{lb:cross} can be directly derived from the following construction.
\begin{lemma}\label{l:lb3}
    For any $0<k<n^2/4$, there exists some DAG $G=(V,E)$ with $n$ nodes such that:
    \begin{enumerate}
        \item there are less than $k$ cross-edges in $G$.
        \item the reachability sign matrix $M$ of $G$ has sign-rank $\Omega{(k^{1/4})}$.
    \end{enumerate}
\end{lemma}

\begin{proof}
    Let $H_t$ be a Sylvester-Hadamard matrix of order $t$. Set $m=\lfloor  \frac{\log_2 k}{2}\rfloor $ and $t=2^m$. Then $t\in (\sqrt{k}/2,\sqrt{k}]$. Since $k<n^2/4$, we have $\sqrt{k}<n/2$, so $n-2t>0$. 
    
    We construct $G=(V,E)$ as follows:
    \begin{enumerate}
        \item $V=X\cup Y\cup R$, where $X=\{x_i\}_{i=1}^t,Y=\{y_j\}_{j=1}^t$ and $R=\{r_s\}_{s=1}^{n-2t}$ are disjoint.
        \item $E=\{(x_i,y_j)|H_{ij}=1,\forall i,j\in [t]\}$, where $H_{ij}$ is the $(i,j)$ entry of $H_t$.
    \end{enumerate}
    Note that $G$ consists of a bipartite directed graph $(X,Y,E)$ and a set $R$ of isolated vertices. 
    
    By the properties of Sylvester-Hadamard matrix shown in Lemma \ref{l:hadamard}, we can know that \begin{enumerate}
        \item $x_1$ can reach every node $y_j$ in $Y$.
        \item For each $i\in [2,t]$, $x_i$ can reach exactly $t/2$ nodes in $Y$.
        \item The spectral norm of $H_{t}$ is $\|H_t\|=\sqrt{t}$.
    \end{enumerate}
    Let $F=(V,E')$, where $E'=\{(x_1,y_j)|j\in [t]\}$. Then $F$ is a spanning forest of $G$. Furthermore, there are $(t-1)t/2$ cross-edges. Since $t\leq \sqrt{k}$, the number of cross-edges is $(t-1)t/2< k$.
    
    Observe that the reachability sign matrix $M$ of $G$ has a sub-matrix $H_{t}$, because we can select the rows of $x_i$ and the columns of $y_j$. By Lemma \ref{l:sign_sub}, $rank_{\pm}M\geq rank_{\pm}H_t$. By Forster’s theorem (refer to Theorem \ref{l:forster}), $H_{t}$ has the sign-rank $rank_{\pm}H_t\geq t/\|H_t\|=\sqrt{t}$. Therefore, we can conclude that $rank_{\pm}M\geq \sqrt{t}=\Omega(k^{1/4})$.
    
\end{proof}

\section{Experiments}
\label{sec:experiments}

We empirically evaluate the representational capacity of our construction based on balanced vertex separators (Section~\ref{sec:tw}). By demonstrating that our algorithm achieves perfect retrieval recall on large-scale real-world hierarchies using low dimensions, we validate our theoretical analysis. These results support our $O(t \log{n})$ bound and indicate that the constant factors involved are modest.

We compare our approach against the ``Handcrafted Embeddings" (which essentially applies JL lemma under the hood) in \citep{you2025hierarchical}. While \citep{you2025hierarchical} also demonstrate a Learned Embedding which optimizes query-document similarity via stochastic gradient descent to achieve high performance in low dimensions, it does not guarantee exact retrieval, requires computationally expensive training, and has no theoretical guarantees. In contrast, our method is faster, constructive, deterministic, and ensures 100\% recall. Furthermore, it is supported by a theoretical guarantee for low treewidth graphs.


\subsection{Implementation Details}
\label{subsec:implementation}

We implement the deterministic construction defined in Section~\ref{sec:tw}, which builds embeddings recursively using balanced vertex separators. To approximate balanced vertex separators, we utilize the KaHIP library~\citep{kahip}. We ensure numerical stability in deep hierarchies by stepwise renormalization of embeddings.

\paragraph{Setup}
We evaluate on three hierarchical datasets: \textit{WordNet} ($N=82$k, DAG) \citep{wordnet}, \textit{Gene Ontology} ($N=24$k, DAG) \citep{go1, go2}, and \textit{Cora} ($N=2.7$k, citation graph --- has some cycles) \citep{cora}. The task is to retrieve all ancestors for a query. We compare against the handcrafted embedding of \citep{you2025hierarchical}, which sums normalized ancestor vectors. Following \citep{you2025hierarchical}, we report \textit{Recall@k}. We display the baseline's performance across a range of dimensions to illustrate its scaling behavior, highlighting the dimension required to attain $>95\%$ recall, consistent with the success threshold established in their work.

\subsection{Results}

Table~\ref{tab:results_comparison} compares our embedding against the handcrafted embedding of \cite{you2025hierarchical}.

\begin{table}[h]
\centering
\caption{Comparison of Recall@k (\%). Bold values denote recall $>95\%$.}
\label{tab:results_comparison}
\resizebox{\textwidth}{!}{%
\begin{tabular}{lclccccc}
\toprule
\multirow{2}{*}{\textbf{Dataset}} & \multirow{2}{*}{\textbf{\shortstack{Ours \\Dim (100\% Recall@k)}}} & \multirow{2}{*}{\textbf{Baseline Method}} & \multicolumn{5}{c}{\textbf{Recall@k by Dimension}} \\
\cmidrule(l){4-8} 
 & & & $64$ & $128$ & $256$ & $512$ & $1024$ \\ 
\midrule
WordNet & \textbf{152} & Handcrafted \cite{you2025hierarchical} & 14.4 & 50.9 & 86.1 & \textbf{99.2} & \textbf{100.0} \\
\midrule
Cora & \textbf{192} & Handcrafted \cite{you2025hierarchical} & 69.5 & 82.0 & 91.2 & \textbf{95.4} & \textbf{96.9} \\ 
\midrule
Gene Ontology & \textbf{492} & Handcrafted \cite{you2025hierarchical} & 22.6 & 53.1 & 82.9 & \textbf{96.8} & \textbf{99.6} \\ 
\bottomrule
\end{tabular}%
}
\end{table}

Our method consistently achieves \textit{100\% Recall} with only a modest number of dimensions. On WordNet, our embedding solves the retrieval task exactly with $d=152$. In comparison, the handcrafted baseline yields only $50.9\%$ recall at a similar dimension ($d=128$) and requires $d=512$ to cross the $95\%$ success threshold—a $3.4\times$ increase in dimensionality. While the Learned (Pretrain-Finetune) baseline \cite{you2025hierarchical} performs significantly better than the handcrafted version (achieving $92.3\%$ recall at $d=64$ on WordNet), it still does not guarantee exact recall. Similarly, on Cora and Gene Ontology, our method guarantees exactness at dimensions lower than those required by the handcrafted baseline. This empirically validates our theoretical bound and demonstrates the efficiency of our constructive approach. Furthermore, our embedding would improve with access to a better vertex separator oracle.

\section*{Acknowledgments}
Piotr Indyk was supported in part by the NSF TRIPODS program (award DMS-2022448).

\bibliographystyle{plainnat}
\bibliography{references}

\appendix

\section{Other Related Works}
\paragraph{Related Works on Graph Parameters.}
The notion of treewidth was introduced by~\citet{robertson1986graph}, as a central tool in analyzing graph structure and designing graph algorithms. A long line of work shows that NP-hard problems can be solved in polynomial time on classes of graphs with bounded treewidth (we refer to the textbook~\citep{cygan2015parameterized}, the survey~\citep{bodlaender11994tourist} or a detailed overview of the state of the art~\citep{korhonen2023improved} for details). 

The problem of efficiently determining reachability in Directed Acyclic Graphs (DAGs) has been extensively studied in the database community. In particular, the ``Tree Cover" methodology starts with a spanning tree and then resolves the remaining reachability information from cross-edges using techniques such as building a look up table; see the works of \cite{agrawal1989efficient, chen2005stack, wang2006dual, jin2011path, zhang2025indexing} for more details. This line of work aims to capture reachability relationships in data structures, with the trade-offs of data structure construction time, query time, and data structure space being studied. In contrast, our work studies embedding reachability structures in low dimensional space.

\section{Additional Preliminaries}\label{sec:additional_prelims}

\subsection{Useful Facts About Depth-First Search}\label{sec:dfs}
We recall the following standard facts about depth-first search (DFS), e.g. see the textbook of \cite{cormen2022introduction}. This underlies the reachability embedding for directed trees in Section~\ref{sec:2}.

\begin{proposition}\label{def:dfs}
    Consider the standard DFS procedure on a directed graph $G=(V,E)$. We have the following observations:
    \begin{enumerate}
        \item If $G$ is connected, the depth-first forest consists of exactly one depth-first tree. 
        \item When DFS on $G$ is finished, each vertex $u\in V$ has been assigned a discovery time $d_u$ (the first time DFS visits $u$) and a finishing time $f_u$ (the time when DFS exists $u$).
        \item All timestamps (discovery and finishing times) are distinct integers.
    \end{enumerate}
\end{proposition}

It is easy to see (and a classical fact) that the discovery and finishing times have the following properties.

\begin{theorem}[Parenthesis theorem]\label{Pare}
    In any DFS of a directed graph $G=(V,E)$, for any two vertices $u$ and $v$, exactly one of the following holds:
    \begin{itemize}
        \item the intervals $[d_u,f_u]$ and $[d_v,f_v]$ are entirely disjoint, and neither $u$ nor $v$ is a descendant of the other in the depth-first forest,
        \item the interval $[d_u,f_u]$ is contained entirely within the interval $[d_v,f_v]$, and $u$ is a descendant of $v$ in a depth-first tree, or
        \item the interval $[d_v,f_v]$ is contained entirely within the interval $[d_u,f_u]$, and $v$ is a descendant of u in a depth-first tree.
    \end{itemize}
\end{theorem}

We also note that prior work implies three dimensions are necessary for a particularly simple directed tree.

\begin{lemma}\label{lem:tree_lb}
    There exists a directed tree $G$ on $O(1)$ nodes which does not admit a reachability embedding of dimension $\le 2$.
\end{lemma}
\begin{proof}
It is shown in \cite{delsarte1989low} and \cite{pmlr-v49-alon16} that for sufficiently large constant $m$, the $m \times m$ sign matrix with diagonals equal to $1$ and off-diagonals equal to $-1$ has sign rank $3$. This sign-matrix is a sub-matrix of the induced reachability sign matrix of the star-graph. The claim follows from noting that the sign-rank is a lower bound on the reachability embedding dimension (see the discussion in Section \ref{sec:lb}).
\end{proof}

\subsection{Useful combinatorial partitioning lemma} \label{App:combpart}

\balaseparestate*

\begin{proof}[Proof of Lemma~\ref{balapart}] \label{combpartproof}
We consider two cases:

\begin{itemize}
    \item \textbf{ Case 1: If $x_i < n/3$ holds for all $i\in [k]$.}  Reorder so that $x_1,\dots,x_k$ are in non-increasing order. Let $j$ be the largest index such that $\sum_{i=1}^j x_i \le 2n/3$ so that $\sum_{i=1}^{j+1} x_i > 2n/3$.
    Set $I_A=\{1,\dots,j\}$ and $I_B=\{j+1,\dots,k\}$. Then
    $$
        S_A=\sum_{i=1}^j x_i\leq 2n/3
    $$
    and
    $$
        S_B = n - S_A \geq n-2n/3 = n/3.
    $$
    Since $x_{j+1}<n/3$ we have
    $$
        S_B = x_{j+1} + \sum_{i=j+2}^k x_i < n/3 + (n - \sum_{i=1}^{j+1} x_i) < n/3 + n/3 = 2n/3.
    $$

    \item \textbf{ Case 2: If exists $i$ with $x_i\geq n/3$.} Let $x_1\geq n/3$.
    Set $I_A=\{1\}$ and $I_B=\{2,\dots,k\}$.
    Then $S_A=x_1\leq n/2$ and $S_B=n-x_1\leq 2n/3$.
\end{itemize}
\end{proof}

\section{Comparing the Treewidth with the Number of Cross-Edges}\label{sec:independence}
In this section, we give examples of graph families that have large treewidth but few total number of cross-edges, and vice-versa, showing that these two parameters are incomparable in general. 






We first need a basic fact about treewidth.
\begin{lemma}[\citet{robertson1986graph}]\label{kn}
    For $n\geq 1$, the complete graph $K_n$ has treewidth $n-1$.
\end{lemma}

We now show our first incomparability result. 
\begin{theorem}\label{in1}
    There exists a  DAG $G$ with $n$ nodes such that:
    \begin{enumerate}
        \item For some spanning tree $T$ of $G$, there is no cross-edge w.r.t. $T$ in $G$.
        \item The underlying undirected graph $G'$ of $G$ has treewidth $\Omega(n)$.
    \end{enumerate}
\end{theorem}

\begin{proof}
    We construct $G=(V,E)$ and $T=(V,E_T)$ as follows:
    \begin{itemize}
        \item $V=\{v_1,v_2,\dots,v_n\}$, 
        \item $E=\{(v_i,v_j)|1\leq i<j\leq n\}$ (i.e. edges point from a smaller label to a larger label),
        \item $E_T=\{(v_i,v_{i+1})|1\leq i\leq n-1\}$.
    \end{itemize}
    Observe that the vertex $v_j$ is reachable from the vertex $v_i$ in $T$ if and only if $i<j$. Since all directed edges in $E$ are from $v_i$ to $v_j$ for $i<j$, there is no cross-edge w.r.t. $T$ in $G$. Note that the underlying undirected graph $G'$ is a complete graph $K_n$. Thus, by Lemma \ref{kn}, $G'$ has treewidth $\Omega(n)$. See the figure below for an example on $6$ nodes where the blue path indicates the spanning tree $T$ and the dashed lines show the cross edges (direction not shown).
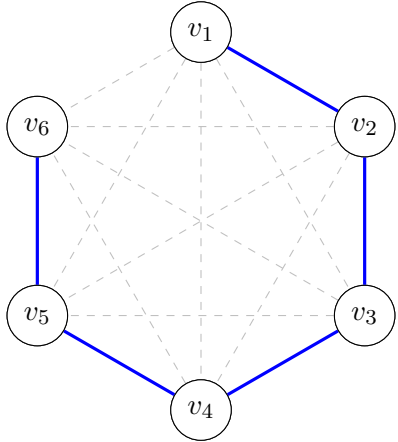
\begin{figure}[!h]
    \centering
  \begin{tikzpicture}
    \def\n{6}         
    \def\radius{2.5cm} 

    \foreach \i in {1,...,\n} {
        \node[circle, draw, fill=white, inner sep=2pt, minimum size=0.8cm] 
            (v\i) at ({90 - (\i-1)*360/\n}:\radius) {$v_{\i}$};
    }

    \foreach \i in {1,...,\n} {
        \foreach \j in {\i,...,\n} {
            \ifnum\i<\j 
                \draw[gray!50, dashed, thin] (v\i) -- (v\j);
            \fi
        }
    }

    \foreach \i [evaluate=\i as \next using int(\i+1)] in {1,...,\numexpr\n-1\relax} {
        \draw[blue, very thick] (v\i) -- (v\next);
    }

    \foreach \i in {1,...,\n} {
        \node[circle, draw, fill=white, inner sep=2pt, minimum size=0.8cm] 
            at ({90 - (\i-1)*360/\n}:\radius) {$v_{\i}$};
    }

\end{tikzpicture}
\end{figure}
\end{proof}

Now we show the complementary result.

\begin{theorem}\label{in2}
    For every sufficiently large $n$, there exists a DAG $G$ with $n$ nodes such that:
    \begin{enumerate}
        \item There exists a spanning tree $T$ of $G$ that has $\Omega(n)$ cross-edges w.r.t. $T$ in $G$.
        \item The underlying undirected graph $G'$ of $G$ has treewidth $O(1)$.
    \end{enumerate}
\end{theorem}

\begin{proof}
We construct such a graph $G = (V, E)$ using a layered structure. Let $k$ be the number of layers, such that the total number of vertices $n$ is linear in $k$.
    \paragraph{Construction}
    Let the layers be $L_1, L_2, \dots, L_k$.
    \begin{itemize}
        \item $L_1 = \{r\}$ (a single source vertex).
        \item For $i = 2, \dots, k$, let $L_i = \{u_i, v_i\}$ (two vertices per layer).
        \item The edge set $E$ consists of all possible edges directed from layer $L_i$ to $L_{i+1}$. Formally, for each $1 \le i < k$, there is a directed edge $(x, y)$ for all $x \in L_i$ and $y \in L_{i+1}$.
    \end{itemize}
    The total number of vertices is $n = 1 + 2(k-1)$, so $k = \Theta(n)$.

    \paragraph{Lower Bound on Cross-Edges}
    Let $T$ be an arbitrary spanning tree of $G$ rooted at $r$. Since every vertex $v \in V \setminus \{r\}$ must have exactly one incoming edge in $T$ (its parent), we can count the tree edges layer by layer. 
    
    Consider the transition between layer $L_i$ and $L_{i+1}$ for $i \ge 2$. The subgraph induced by $L_i \cup L_{i+1}$ is a complete bipartite graph directed from $L_i$ to $L_{i+1}$, containing $|L_i| \times |L_{i+1}| = 2 \times 2 = 4$ edges.  To include the vertices of $L_{i+1}$ in the spanning tree, $T$ must select exactly 2 edges (one parent for $u_{i+1}$ and one for $v_{i+1}$). Consequently, $ 2$ edges between these layers are not in $T$.

    In a layered DAG where edges only exist between $L_i$ and $L_{i+1}$, any edge $(x, y)$ spans exactly one layer. Therefore, $x$ cannot be an ancestor of $y$ unless $x$ is the immediate parent of $y$. If $(x, y)$ is a non-tree edge, $x$ is not the parent of $y$, and thus $x$ is not an ancestor of $y$. Similarly, $y$ cannot be an ancestor of $x$ (due to edge directions). Therefore, all non-tree edges are cross-edges, as desired. Summing over all layers $i=2$ to $k-1$, the total number of cross-edges is at least $2(k-2)$, which is $\Omega(n)$.

    \paragraph{Treewidth}
    We show that $G'$ (the underlying undirected graph) has treewidth $O(1)$ by constructing a tree decomposition (specifically, a path decomposition). Define the bags $B_1, \dots, B_{k-1}$ as:
    \[ B_i = L_i \cup L_{i+1}. \]
It is easy to check that this decomposition satisfies all properties of Definition \ref{def:tree_cover}, and the size of each bag is bounded by a constant, implying that the treewidth is $O(1)$. See the figure below for a pictorial representation of the graph.

\begin{figure}[!h]
\centering
\begin{tikzpicture}[
    node distance=2.5cm,
    layer node/.style={
        circle, 
        draw=black, 
        very thick, 
        minimum size=0.8cm, 
        fill=white,
        font=\bfseries
    },
    edge/.style={
        ->, 
        >={Stealth[round]}, 
        thick, 
        draw=gray!80
    }
]

    \def\numlayers{4} 
    
    
    \node[layer node] (L0-1) at (0,0) {$r$};

    \foreach \i in {1,...,\numlayers} {
        \pgfmathsetmacro{\xpos}{\i * 2.5}
        
        \node[layer node] (L\i-1) at (\xpos, 1) {$L_{\i,1}$};
        \node[layer node] (L\i-2) at (\xpos, -1) {$L_{\i,2}$};
    }


    \foreach \target in {1,2} {
        \draw[edge] (L0-1) -- (L1-\target);
    }

    \pgfmathsetmacro{\penultimate}{\numlayers - 1}
    
    \foreach \i in {1,...,\penultimate} {
        \pgfmathsetmacro{\nextlayer}{int(\i + 1)}
        
        \foreach \source in {1,2} {
            \foreach \target in {1,2} {
                \draw[edge] (L\i-\source) -- (L\nextlayer-\target);
            }
        }
    }

\end{tikzpicture}
\end{figure}
\end{proof}
    

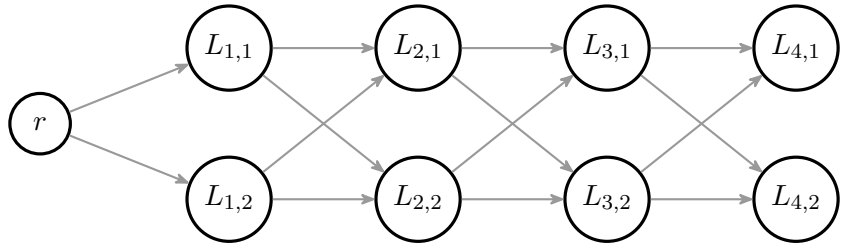
\end{document}